\pdfoutput=1

\documentclass[11pt]{article}

\usepackage{emnlp2021}

\usepackage{times}
\usepackage{latexsym}

\usepackage[T1]{fontenc}

\usepackage[utf8]{inputenc}

\usepackage{microtype}

\usepackage{comment}
\usepackage{amssymb,amsthm}
\usepackage{url}
\usepackage{lipsum}
\usepackage{listings}
\usepackage{graphicx}
\usepackage{listings}
\usepackage{footnote}
\usepackage{float}
\usepackage{subfiles}

\usepackage{algorithm2e}
\usepackage{commath,booktabs}
\DeclareMathOperator*{\argmax}{arg\,max}
\DeclareMathOperator*{\argmin}{arg\,min}

\newtheorem{prop}{Proposition}
\newtheorem{theorem}{Theorem}
\newtheorem{lemma}{Lemma}

\title{A Differentiable Relaxation of Graph Segmentation \\  and Alignment  for AMR Parsing} 

\author{ {Chunchuan Lyu}$^1$  ~~ {Shay B. Cohen}$^1$ ~~ {Ivan Titov}$^{1,2}$ 
 \smallskip \\
      $^1$ {ILCC, School of Informatics, University of Edinburgh}  \\
    $^2$ {ILLC, University of Amsterdam}
 \smallskip \\
 {\tt chunchuan.lv@gmail.com} ~~ {\tt scohen@inf.ed.ac.uk} ~~ {\tt ititov@inf.ed.ac.uk}  \\
}

\date{}

\usepackage{xr}
\makeatletter
\newcommand*{\addFileDependency}[1]{
  \typeout{(#1)}
  \@addtofilelist{#1}
  \IfFileExists{#1}{}{\typeout{No file #1.}}
}
\makeatother

\newcommand*{\myexternaldocument}[1]{%
    \externaldocument{#1}%
    \addFileDependency{#1.tex}%
    \addFileDependency{#1.aux}%
}
\myexternaldocument{supp}

\begin{document}

\setlength{\abovedisplayskip}{1pt}
\setlength{\belowdisplayskip}{1pt}
\setlength{\abovedisplayshortskip}{1pt}
\setlength{\belowdisplayshortskip}{1pt}
\maketitle
\begin{abstract}
Abstract Meaning Representations (AMR) are a broad-coverage semantic formalism which represents sentence meaning as a directed acyclic graph. To train most AMR parsers, one needs to segment the graph into subgraphs and align each such subgraph to a word in a sentence; this is normally done at preprocessing, relying on hand-crafted rules. In contrast, we treat both alignment and segmentation as latent variables in our model and induce them as part of end-to-end training.
 As marginalizing over the structured latent variables is infeasible, we use the variational autoencoding framework. 
 To ensure end-to-end differentiable optimization, we introduce a differentiable relaxation of the segmentation and alignment problems.
We observe that inducing segmentation yields substantial gains over using a `greedy' segmentation heuristic. The performance of our method also approaches that of a model that relies on the segmentation rules of \citet{lyu-titov-2018-amr}, which were hand-crafted to handle individual AMR constructions.
\end{abstract}

\section{Introduction}

Abstract Meaning Representation (AMR; \citealt{Banarescu13abstractmeaning}) is a broad-coverage semantic formalism which represents sentence meaning as rooted labeled directed acyclic graphs. 
The representations have been exploited in a wide range of 
tasks, including text summarization~\cite{Liu2015TowardAS,dohare2017text,Hardy2018GuidedNL}, machine translation~\cite{Jones2012SemanticsBasedMT,Song2019SemanticNM}, paraphrase detection~\cite{Issa2018AbstractMR} and question answering~\cite{Mitra2016AddressingAQ}. 

An AMR graph can be regarded as consisting of 
multiple concept subgraphs, which can be individually
aligned to sentence tokens~\cite{Flanigan2014ADG}.  
 In Figure~\ref{fig:amr-example}, each dashed box represents the boundary of a single semantic subgraph. Red arrows represent the alignment between subgraphs and tokens.
 For example, `(o / opine-01: ARG1 (t / thing))' refers to a combination of the predicate `opine-01' and a filler of its semantic role ARG1.  Intuitively, this subgraph needs to be aligned to the token `opinion'. Similarly, `(b / boy)' should be aligned to the token `boy'. Given such an alignment and segmentation, it is straightforward to construct a simple  parser: parsing can be framed as tagging input tokens with subgraphs (including empty subgraphs), followed by predicting relations between the subgraphs. 
 The key obstacle to training such an AMR parser is that the segmentation
 and alignment between AMR subgraphs and words are latent, i.e. not annotated in the data.
 

\begin{figure}[t!]
\centering
\includegraphics[width=0.8\columnwidth]{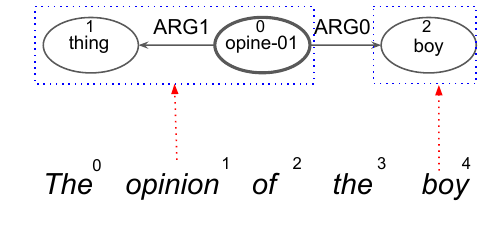}
\vspace{-2ex}
\caption{An example of AMR, the dashed red arrows mark latent alignment. Dashed blue boxes represent the latent segmentation.} 
\label{fig:amr-example}
\end{figure}%

Most previous work adopts a pipeline approach to handling this obstacle. They rely on a pre-learned aligner~(e.g., \cite{pourdamghani-etal-2014-aligning}) to produce the alignment, and apply a rule system to segment the AMR subgraph~\cite{Flanigan2014ADG,Werling2015RobustSG,Damonte2017AnIP,Ballesteros2017AMRPU,peng-etal-2015-synchronous,Artzi2015BroadcoverageCS,Groschwitz2018AMRDP}. 
While ~\newcite{lyu-titov-2018-amr} jointly optimize the parser and the alignment model, 
the rules handling specific constructions  still needed to be crafted to segment the graph. 
The segmentation rules are relatively complex --- e.g.,  the rules of \newcite{lyu-titov-2018-amr} targeted 40 different
AMR subgraph types --- and language-dependent. AMR has never been intended to be used as an interlingua~\cite{Banarescu13abstractmeaning,damonte-18} 
and AMR banks for individual languages substantially diverge from English AMR. For example, Spanish AMR represents
pronouns and ellipsis differently from the English one~\cite{migueles-abraira-etal-2018-annotating}.
 As new AMR sembanks in languages other than English are being  developed~\cite{anchieta-pardo-2018-towards,song-etal-2020-construct},
 domain-specific AMR extensions get developed \cite{bonn-etal-2020-spatial,bonial-etal-2020-dialogue},
 and extra constructions are getting introduced to AMRs~\cite{bonial-etal-2018-abstract}, eliminating the need for rules while learning graph segmentation from scratch is becoming an important problem to solve. 

We propose to optimize a graph-based parser that treats the alignment and graph segmentation  as latent variables. The graph-based parser consists of two parts: concept identification and relation identification. The concept identification model generates the AMR nodes, and the relation identification component decides on the labeled edges. During training, both components rely on latent alignment and segmentation, which are being induced simultaneously.  Importantly, at test time, the parser simply tags the input with the subgraphs and predicts the relations, so there is no test-time overhead from using the latent-structure apparatus. An extra benefit of this approach, in contrast to encoder-decoder AMR models~\cite{Konstas2017NeuralAS,Noord2017NeuralSP,Cai2020AMRPV} is its transparency, as one can readily see which input token triggers each subgraph.\footnote{The code is available at \url{https://github.com/ChunchuanLv/graph-parser}.}

To develop our parser,  we frame the alignment and segmentation problems as choosing a {\it generation order} of concept nodes, as we explain in Section~\ref{sec:generation_order}.  As marginalization over the latent generation orders is infeasible, we adopt the variational auto-encoder (VAE) framework~\cite{kingma2013auto}. 
Intuitively, a trainable neural module (an encoder in the VAE) is used to sample a plausible generation order (i.e., a segmentation plus an alignment), which is then used to train the parser (a decoder in the VAE). As
one cannot `differentiate through' a sample of discrete variables to train the encoder, we introduce a differentiable relaxation which makes our objective end-to-end differentiable.

We experiment on the AMR 2.0 and 3.0 datasets. 
We compare to a {\it greedy segmentation} heuristic, inspired by ~\citet{Naseem2019RewardingST}, that produces a segmentation deterministically and provides a strong baseline to our segmentation induction method. We  also use a version of our model with segmentation induction replaced by a hand-crafted rule-based segmentation system from previous work;\footnote{
We adopt the code from ~\newcite{Zhang2019AMRPA}, which is an extension of the system introduced by ~\newcite{lyu-titov-2018-amr}, and also used by ~\newcite{Cai2020AMRPV}.} it can be thought of as an upper bound on how well induction can work. On AMR 2.0 (LDC2016E25), we found that our VAE system obtained a competitive Smatch score of 76.1, reducing the gap between using the segmentation heristic (75.2) and the rules exploiting the prior knowledge about AMR  (76.8). On AMR 3.0 (LDC2020T02), the VAE system gets even closer to the rule-based system (75.5 vs 75.7), possibly because the rules were designed for AMR 2.0. 
Our main contributions are:
\begin{itemize}
\item we frame the alignment and segmentation problems as inducing a generation order, and provide a continuous relaxation to this discrete optimization problem; 
\item we empirically show that our method outperforms a strong heuristic baseline and approaches the performance of a complex hand-crafted rule system. 
\end{itemize}  

Our method makes very few assumptions about the nature of the graphs,
so it may be effective in other tasks that can be framed as graph prediction  (e.g., executable semantic parsing, \citealt{liang2016learning}, or scene graph
prediction, \citealt{xu2017scene}).



\section{Casting Alignment and Segmentation as Choosing a Generation Order} 
\subsection{Preliminaries}
\label{sec:seg_notation}

We start by introducing the basic concepts and notation. We refer to words in a sentence as $\mathbf{x} = (x_0,\ldots,x_{n-1})$, where $n$ is the sentence length. 
 The concepts (i.e. labeled nodes) are $\mathbf{v} = (v_0,v_1, \ldots, v_m)$, where $m$ is the number of concepts.
In particular, $v_m= \emptyset $ denotes a dummy terminal node; its purpose will be clear in Section \ref{sec:generation_order} where we will define the generative model.  We refer to all nodes, except for the terminal node ($\emptyset$), as concept nodes. 

A relation between `predicate concept' $i$
and `argument concept' $j$ is denoted by $\mathbf{E}_{ij}$.  
It is set to $\emptyset $ if $j$ is not an argument of $i$.  We will use $\mathbf{E}$ to denote all edges (i.e. relations) in the graph. In addition, we refer to the whole AMR graph as $\mathbf{G}=(\mathbf{v},\mathbf{E})$. 

Our goal is to associate each input token with a (potentially empty) subset of the concept nodes in the AMR graph, while making sure that we get a {\it partition} of the node set. In other words, each node in the original AMR graph belongs to exactly one subset. In that way, we deal with both segmentation and alignment. 
Each subset uniquely corresponds to a {\it vertex-induced subgraph} (i.e., the subset of nodes together with any edges whose both endpoints are in this subset). For this reason, we will refer to the problem
as graph decomposition\footnote{We slightly abuse the terminology as, in graph theory,  graph decomposition usually refers to a partition of edges -- rather than nodes -- of the original graph.} 
and to each subset as a subgraph. 
We will explain how we deal with edges of the AMR graph in Section~\ref{subsect:rel}.

\subsection{Generation Order}  \label{sec:generation_order}
\begin{figure}[t!]
\centering
\includegraphics[width=0.8\linewidth]{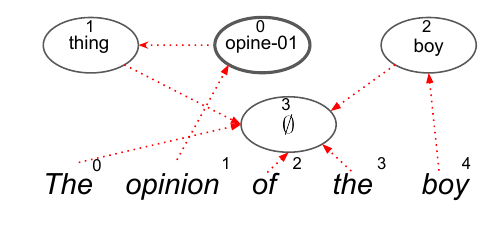}
\vspace{-2ex}
\caption{AMR concept identification model generates nodes following latent generation order at training time.} 
\label{fig:amr-order}
\centering
\includegraphics[width=0.8\linewidth]{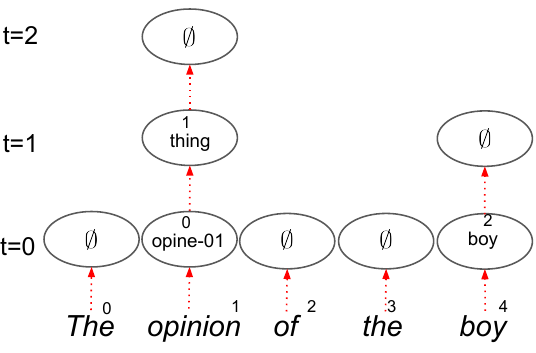}
\vspace{-2ex}
\caption{At test time, the AMR concept identification model generates 
nodes autoregressively, starting from each sentence token. Importantly, it is just an `unrolled' form of the order shown in Figure~\ref{fig:amr-order}.} 
\label{fig:seg_testing}
\end{figure}%


We choose a subset of nodes for each token by assigning an order to which the nodes are selected for such subset.  In Figure~\ref{fig:amr-order}, dashed red arrows point from every node to the subsequent node to be selected.  For example, given the word `opinion', the node `opine-01' is chosen first, and then it is followed by another node `thing'. After this node, we have an arrow pointing to the node $\emptyset $, signifying that we finished generating nodes aligned to the word `opinion'. We refer to these red arrows as a {\it generation order}. 

A generation order determines a graph decomposition. To recover it from a generation order, we assign connected nodes (excluding the terminal node) to the same subgraph. Then, a subgraph will be aligned to the token that generated those nodes. In our example, `opine-01' and `thing' are connected, and, thus, they are both aligned to the word `opinion'. The alignment is encoded by arrows between tokens and concept nodes, while the segmentation is represented by arrows between concept nodes.

From a modeling perspective, the nodes will be generated with an autoregressive model,
which is easy to use at test time  (Figure~\ref{fig:seg_testing}). From each token, a chain of nodes is generated until the stop symbol $\emptyset $  is predicted.  It is more challenging to see how to induce the order and train the autoregressive model at the same time; we will discuss this in Sections~\ref{sec:main}~and~\ref{sec:latent}.

\noindent {\bf Constraints} \quad
While in Figure~\ref{fig:amr-order} the red arrows determine a 
valid generation order, in general, the arrows have to obey certain constraints. Formally, we denote alignment by $\mathbf{A} \in  \{0,1\}^{ n \times (m +1)}$, where $\mathbf{A}_{ki}=1$ means that for token $k$ we start by generating node $i$. As the token can only point to one node, we have a constraint $\sum_i \mathbf{A}_{ki}=1$. 
Similarly, for a segmentation $\mathbf{S} \in  \{0,1\}^{ m \times (m +1)} $ we have a constraint $ \sum_{j}\mathbf{S} _{ij} = 1 $. Setting $\mathbf{S}_{ij} =1 $ indicates that node $i$ is followed by node $j$.  In Figure~\ref{fig:amr-order}, 
we have $\mathbf{A}_{03}=\mathbf{A}_{10}=\mathbf{A}_{23} = \mathbf{A}_{33}=\mathbf{A}_{42} = 1$ and $\mathbf{S}_{01}=\mathbf{S}_{13}=\mathbf{S}_{23} = 1$; the rest is $0$.
Now, we have the full generation order as their concatenation $ \mathbf{O}= [\mathbf{A} ; \mathbf{S}] \in  \{0,1\}^{ (n+m) \times (m +1)}$. As one node can only be generated once (except for $\emptyset$), we have a joint constraint: $\forall j \neq m, \sum_{l} \mathbf{O}_{lj} =1$.  Furthermore, the graph defined by $\mathbf{O}$ should be acyclic, as it represents the generative process. We denote the set of all valid generation orders as $\mathcal{O}$. 
In the following sections, we will discuss how this generation order is used in the model and how to infer it as a latent variable while enforcing the above constraints. 
\section{Our Model}\label{sec:main}
Formally, we aim at estimating $P_\theta(\mathbf{v},\mathbf{E}|\mathbf{x})$, the likelihood of an AMR graph given the sentence. Our graph-based parser is composed of two parts: concept identification $P_\theta(\mathbf{v}|\mathbf{x},\mathbf{O})$ and relation identification $P_\theta(\mathbf{E}|\mathbf{x},\mathbf{O},\mathbf{v})$. The concept identification model generates concept nodes, and the relation identification model assigns relations  between them. Both require the latent generation order at the training time,  denoted by $\mathbf{O}$. 
Overall, we have the following objective:
\begin{align}
   & \log P_\theta(\mathbf{v},\mathbf{E}|\mathbf{x})   \\
    = &\log \sum_{\mathbf{O} } {P_\theta(\mathbf{O}) P_\theta(\mathbf{v}|\mathbf{x},\mathbf{O}) P_\theta(\mathbf{E}|\mathbf{x},\mathbf{O},\mathbf{v}) } ~\label{eq:joint_margin},
\end{align}
where $P_\theta(\mathbf{O})$ is a prior on the generation orders, discussed in Section~\ref{sec:soft}. 
To efficiently optimize this objective end-to-end, as will be discussed in Section~\ref{sec:latent},
we need to ensure that both concept and relation identification models admit relaxation, i.e., they should be well-defined for real-valued $\mathbf{O}$.  


In the following subsections, we go through concept identification, relation identification, and their corresponding relaxations.

\subsection{Concept Identification}

\begin{figure}[t!]
\centering
\includegraphics[width=0.9\linewidth]{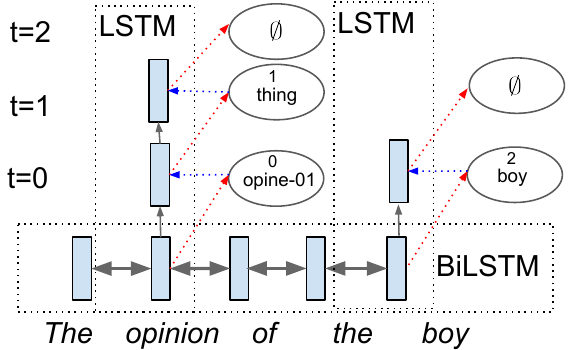}
\caption{AMR concept identification model runs several independent LSTMs to generate nodes autoregressively at test time.} 
\label{fig:seg_lstm}
\end{figure}%

 As shown in Figure~\ref{fig:seg_lstm}, our neural model first encodes the sentence with $\mathrm{BiLSTM}$, producing token representations  $\mathbf{h}^{\mathrm{token}}_k$ ($k \in [0,\ldots n-1]$), 
  then generates nodes autoregressively at each token with another $\mathrm{LSTM}$. 
 
 In training, we need to be able to run the models with any potential generation order and compute $P_\theta(\mathbf{v}|\mathbf{x},\mathbf{O})$. 
 If we take the order defined in Figure~\ref{fig:amr-order},
 the node 1 (`thing') is predicted 
 relying on the corresponding hidden representation;
 we refer to this representation as  $\mathbf{h}^{\mathrm{node}}_1$ where is $1$ is the node index. With the discrete generation order  defined by red arrows in Figure~\ref{fig:amr-order}, $\mathbf{h}^{\mathrm{node}}_1$  is just the LSTM state
 of its parent (i.e. `opine-01'). However, to admit relaxations, our computation should be well-defined when the generation order $\mathbf{O}$ is soft (i.e. attention-like). 
 In that case, $\mathbf{h}^{\mathrm{node}}_1$ will be a weighted sum of LSTM representations of other nodes and input tokens, where the weights are defined by $\mathbf{O}$.
 Similarly, the termination symbol $\emptyset$ for the token `opinion' is predicted from its hidden representation; 
 we refer to this representation as  $\mathbf{h}^{\mathrm{tail}}_1$, where $1$ is the position of `opine' in the sentence. With the hard generation order of Figure~\ref{fig:amr-order}, $\mathbf{h}^{\mathrm{tail}}_1$ is just the LSTM state computed after choosing the preceding node (i.e. `thing'). In the relaxed case, it will again be a weighted sum with the weights defined by $\mathbf{O}$.
 
\label{sec:seg_concept_generation}
Formally, the probability of concept identification step can be decomposed into probability of generating $m$ concepts nodes and $n$ terminal nodes (one for each token):
\begin{align}
    P_\theta(\mathbf{v}|\mathbf{x},\!\mathbf{O}\!)\!\!=\!\!\prod_{i=0}^{m-1}\!\!P_\theta(\mathbf{v}_i|\mathbf{h}^{\mathrm{node}}_i) \!  
     \prod_{k=0}^{n-1}\!P_\theta(\emptyset |\mathbf{h}^{\mathrm{tail}}_k)
    \label{eq:prob}
\end{align}
Representation
$\mathbf{h}^{\mathrm{node}}_i$ is computed as the weighted sum of the LSTM states of preceding nodes as defined by $\mathbf{O}$ (recall that $\mathbf{O} = [ \mathbf{A} ; \mathbf{S} ]$):
\begin{align}
\nonumber
 \mathbf{h}^{\mathrm{node}}_i := 
 &  \sum_{j=0}^{m-1}  \mathbf{S}_{ji}\mathrm{LSTM}{(\mathbf{h}^{\mathrm{node}}_j,\mathbf{v}_j)} \\
     & +  \sum_{k=0}^{n-1} \mathbf{A}_{ki}\mathbf{h}^{\mathrm{token}}_k. \label{eq:concept_input}
\end{align}
Note that the preceding node can be either a concept node (then the output of the LSTM, consuming the preceding node, is used)
or a word (then we use its contextualized encoding). The first term in Equation~\ref{eq:concept_input} corresponds to the former situation, and the second one to the latter.   

Note that this expression is `recursive' -- each node's representation $\mathbf{h}^{\mathrm{node}}_i$ is computed based on representations of all the nodes $\mathbf{h}^{\mathrm{node}}_j$; $i, j \in 1, \ldots m-1$. 
 Iterating  the assignment defined by Equation~\ref{eq:concept_input} 
  for a valid {\bf discrete} generation order (i.e., a DAG, like the one given in Figure~\ref{fig:amr-order}), 
 will converge to a stationary point. Crucially, in this discrete case, the stationary point will be equal to the result of applying the autoregressive model (as used in test time, see  Figure~\ref{fig:seg_lstm}). The stationary point will be reached after $T$ steps, where $T$ is the number of nodes in the largest subgraph.\footnote{We use  $T=4$, as we do not expect subgraphs with more than 4 nodes.} 
 This `message passing' process is fully differentiable and, importantly, well-defined for a relaxed generation order where $\mathbf{A}_{ki}$ and  $\mathbf{S}_{ji}$ are non-binary.  The equivalence between the train-time message passing and the test-time autoregressive computation with discrete $\mathbf{O}$
prevents the gap between training and testing, as long as the optimization  converges to a near-discrete solution.

The representations $\mathbf{h}^{\mathrm{tail}}_k$,
needed for the terms
$P_\theta(\emptyset |\mathbf{h}^{\mathrm{tail}}_k)$ 
in Equation~\ref{eq:prob}, are computed as:
\begin{align}
\nonumber
 \mathbf{h}^{\mathrm{tail}}_k = &
  \sum_{j=0}^{m-1}  \mathbf{B}_{jk}\mathrm{LSTM}{(\mathbf{h}^{\mathrm{node}}_j,\mathbf{v}_j)} \\
      & +  (1 - \sum_{j=0}^{m-1}{ \mathbf{B}_{jk} })
     \mathbf{h}^{\mathrm{token}}_k, \label{eq:tail_comp}
\end{align}
where $\mathbf{B}_{jk}=1$  denotes that the concept node $j$ is the last concept node before generating $\emptyset$ for the token $k$, else $\mathbf{B}_{jk}=0$. E.g. in Figure~\ref{fig:amr-order}, we have $\mathbf{B}_{11}=\mathbf{B}_{42}=1$, and others are 0.
Again, in the discrete case, the result will be exactly equivalent to what is obtained by running the corresponding autoregressive model (as in test time,   Figure~\ref{fig:seg_lstm}), but the computation is also well-defined and differentiable in the relaxed cases, where $B_{jk}$ are real-valued.

While it is clear how $S_{ji}$ and $A_{ki}$ in Equation~\ref{eq:concept_input} and $B_{jk}$ in Equation~\ref{eq:tail_comp} are defined with discrete $\mathbf{O}= [\mathbf{A}; \mathbf{S}]$, we show
how they can defined with relaxed (non-binary) $\mathbf{O}$ in Appendix~\ref{append:concept}. The MLPs used to compute $P_\theta(\mathbf{v}_i|\mathbf{h}^{\mathrm{node}}_i) $ and $P_\theta(\emptyset|\mathbf{h}^{\mathrm{tail}}_i)$ are also defined there. 


\subsection{Relation Identification}
\label{subsect:rel}
Similarly to~\newcite{lyu-titov-2018-amr}, we use an arc-factored model for relation identification (i.e. predicting AMR edges): 
\begin{align}
&P_\theta(\mathbf{E}|\mathbf{x},\mathbf{O},\mathbf{v}) = \prod_{i,j=1}^m P_\theta(\mathbf{E}_{ij}|\mathbf{h}^{\mathrm{edge}}_i,\mathbf{h}^{\mathrm{edge}}_j)  
\end{align}
where $P_\theta(\mathbf{E}_{ij}|\mathbf{h}^{\mathrm{edge}}_i,\mathbf{h}^{\mathrm{edge}}_j)  $ is the softmax of the biaffine function of node representations  $\mathbf{h}^{\mathrm{edge}}_i$ and $\mathbf{h}^{\mathrm{edge}}_j$.
The  node representations are defined as \begin{align}\mathbf{h}^{\mathrm{edge}}_i=
 \mathrm{NN}^{\mathrm{edge}}(\mathbf{h}^{\mathrm{node}}_i\circ \sum_{k=0}^{n-1}\mathbf{A}^\mathrm{\infty}_{ki}\mathbf{h}^{\mathrm{token}}_k),\end{align}
where $\circ$ denotes concatenation, $\mathbf{h}^{\mathrm{node}}_i$ is defined in section ~\ref{sec:seg_concept_generation}, and $\mathbf{A}^\mathrm{\infty}_{ki}$ determines whether node $i$ is in a subgraph aligned to token $k$ or not. Note that this is different from $\mathbf{A}_{ki}$ which encodes that the node $i$ is the first node in the subgraph (e.g., in Figure~\ref{fig:amr-order},
$\mathbf{A}_{11} =0$ but $\mathbf{A}^\mathrm{\infty}_{11}=1$). In the continuous case, as  used during training, $\mathbf{A}^\mathrm{\infty}_{ki}$  can be thought of as the alignment probability that can be computed
from $\mathbf{O}$ (see Appendix~\ref{appendix:align_derivative}). 

\section{Estimating Latent Generation Order} \label{sec:latent}

\begin{figure}[t!]
\centering
\includegraphics[width=0.7\linewidth]{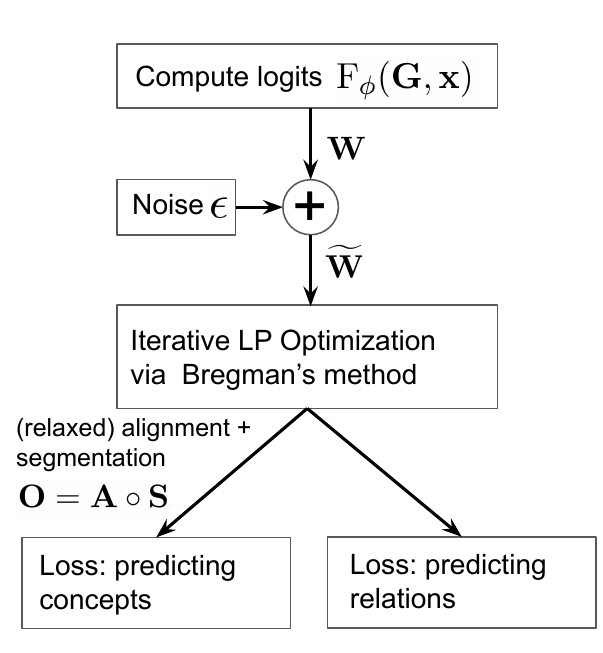}
\vspace{-2ex}
\caption{Overview of the computation graph.}
\label{fig:comp-graph}
\end{figure}%

We show how to estimate the latent generation order jointly with the parser, as also illustrated in Figure~\ref{fig:comp-graph}. 

\subsection{Variational Inference}
In Equation~\ref{eq:joint_margin}, marginalization over $\mathbf{O}$ is intractable due to the use of neural parameterization in $P_\theta(\mathbf{v}|\mathbf{x},\mathbf{O})$ and $P_\theta(\mathbf{E}|\mathbf{x},\mathbf{O},\mathbf{v})$. 
Instead, we resort to the variational auto-encoder (VAE) framework~\cite{kingma2013auto}. VAEs optimize a lower bound on the   marginal likelihood:
\begin{align}
   &\log \sum_{\mathbf{O} } P_\theta(\mathbf{O}) P_\theta(\mathbf{v}|\mathbf{x},\mathbf{O}) P_\theta(\mathbf{E}|\mathbf{x},\mathbf{O},\mathbf{v})   \nonumber \\
    \geq & \underset{\mathbf{O} \sim Q_\phi(\mathbf{O}|\mathbf{G},\mathbf{x}) }{\mathrm{E}} \log P_\theta(\mathbf{v}|\mathbf{x},\mathbf{O})P_\theta(\mathbf{E}|\mathbf{x},\mathbf{O},\mathbf{v}) \nonumber \\
    &- \mathrm{KL}(Q_\phi(\mathbf{O}|\mathbf{G},\mathbf{x})|| P_\theta(\mathbf{O})) ~\label{eq:elbo},
\end{align}
where $\mathrm{KL}$ is the  KL divergence, and  $Q_\phi(\mathbf{O}|\mathbf{G},\mathbf{x})$ (the encoder, aka the inference network) is a distribution parameterized with a neural network. The lower bound is maximized with respect to both the original parameters $\theta$ and the variational parameters $\phi$. The distribution $Q_\phi(\mathbf{O}|\mathbf{G},\mathbf{x})$ can be thought of as an approximation to the intractable posterior distribution
$P_\theta(\mathbf{O}|\mathbf{G},\mathbf{x})$.

\subsection{Stochastic Softmax}~\label{sec:soft}
In order to estimate the gradient
with respect to the encoder parameters $\phi$,
we use the perturb-and-MAP framework~\cite{Papandreou2011PerturbandMAPRF,hazan2012partition},
specifically the stochastic softmax~\cite{Paulus2020GradientEW}, which is
 a generalization the Gumbel-softmax trick~\cite{Jang2016CategoricalRW,Maddison2017TheCD} to the structured case.

With Stochastic Softmax, instead of sampling $\mathbf{O}$ directly, %
we independently compute logits $\mathbf{W} \in \mathbb{R}^{(n+m)\times(m+1)}$ for  all the potential edges in the generation order, and perturb them:
  \begin{align} \label{eq:logits}
   \mathbf{W} &= \mathrm{F}_\phi(\mathbf{G},\mathbf{x}) \\
  \widetilde{ \mathbf{W}} &=  \mathbf{W} + \mathbf{\epsilon},  \quad  \textrm{where} \quad \mathbf{\epsilon}_{ij} \sim  \mathcal{G}(0,1)
\end{align}
where $\mathrm{F}_\phi$ is a neural module computing the logits (see Section~\ref{sec:mask}), $\mathcal{G}(0,1)$ is the standard Gumbel distribution, and $\epsilon \in \mathbb{R}^{(n+m)\times(m+1)}$. Then, those perturbed logits $\widetilde{ \mathbf{W}}$ are fed into a constrained convex optimization problem:
  \begin{align} \label{eq:reg}
  &\mathbf{O}(\widetilde{ \mathbf{W}},\tau) := \argmax_{  \mathbf{O} \geq 0} \langle \widetilde{ \mathbf{W}}, \mathbf{O} \rangle -\tau \langle \mathbf{O}, \log \mathbf{O} \rangle \nonumber \\
    \mbox{s.t.}  &   \forall j < m \sum_{i=0}^{n+m-1} \mathbf{O}_{ij} =1 ;\forall i \sum_{j=0}^{m} \mathbf{O}_{ij} =1
\end{align}
This is a linear programming (LP) relaxation of constraints  discussed in Section~\ref{sec:generation_order}, where we permit continuous-valued  $\mathbf{O}$. Importantly, this LP relaxation is `tight', and 
ensures that
 $\mathbf{O}(\widetilde{ \mathbf{W}},0) $ is a valid generation order.\footnote{See proof in Appendix~\ref{append:valid}.
 } 

Now, as we will show in the next section,
the solution to this optimization 
$\mathbf{O}(\widetilde{ \mathbf{W}},\tau) $
can be obtained with a differentiable computation,
thus, we write:
\begin{align}
    \mathbf{O}_\phi(\mathbf{\epsilon},\mathbf{G},\mathbf{x}) =\mathbf{O}(\widetilde{ \mathbf{W}},\tau)
\end{align}  
The entropy regularizor, weighted by $\tau>0$ (`the temperature'), ensures differentiability  with respect to $\mathbf{W}$ and, thus, with respect to $\phi$, as needed to train the encoder. 

We still need to handle the $\mathrm{KL}$ term in Equation~\ref{eq:elbo}. We define the prior probability $P_\theta(\mathbf{O})$ implicitly by having $\mathbf{W}=0$ in the stochastic softmax framework. Even then, $\mathrm{KL}(Q_\phi(\mathbf{O}|\mathbf{G},\mathbf{x})|| P_\theta(\mathbf{O}))$ cannot be easily computed. Following~\newcite{Mena2018LearningLP}, we upper bound it by replacing it with $\mathrm{KL}(\mathcal{G}(  \mathbf{W},1)||\mathcal{G}(0,1))$, which is available in closed form. 

\subsubsection{Bregman's Method}
To optimize objective (\ref{eq:reg}) we iterate over the following steps of optimization: \begin{align}
    \mathbf{O}^{(0)} &= \exp \frac{ \widetilde{\mathbf{W}}}{\tau}  \label{eq:init}  \\
\forall j < m,\; \mathbf{O}^{(t+\frac{1}{2})}_{:,j} & =  \mathcal{T}(\mathbf{O}^{(t)}_{:,j}) \label{eq:column}\\
  \mathbf{O}^{(t+\frac{1}{2})}_{:,m}  &=  \mathbf{O}^{(t)}_{:,m}  \label{eq:fix} \\
 \forall i,\; \mathbf{O}^{(t+1)}_{i,:} &= \mathcal{T}( \mathbf{O}^{(t+\frac{1}{2})}_{i,:}) \label{eq:row} 
\end{align}
where  $\{i,:\}$ index $i$th row, $\{:,j\}$ index jth column and $ \mathcal{T} = \frac{\mathbf{x}}{\sum_{i}\mathbf{x}_i}$ normalize the vectors. Intuitively, the alignment scores are initially computed from the logits $\widetilde{\mathbf{W}}$, without taking constraints into account, and then alternating optimization is used to `fit' the constraints on columns and rows.

\begin{prop}\label{prop:order-converge}
$\lim_{t\rightarrow \infty}   \mathbf{O}^{(t)} = \mathbf{O}( \widetilde{ \mathbf{W}},\tau) $ where $\mathbf{O}( \widetilde{ \mathbf{W}},\tau) $ is defined in Equation~\ref{eq:reg}.
\end{prop}
See Appendix~\ref{prof:order-converge} for a proof based on the proof for the Bregman method~\cite{Bregman1967TheRM}. 
In practice, we take $T=50$, and have $  \mathbf{O}_\phi(\mathbf{\epsilon},\mathbf{G},\mathbf{x})= \mathbf{O}^{(T)}$. Importantly, this algorithm is highly parallelizable and amendable to batch implementation on GPU. We compute the gradients with unrolled optimization. 



\subsubsection{Neural Parameterization} \label{sec:mask}
We introduce the neural modules used for estimating logits $ \mathbf{W} = \mathrm{F}_\phi(\mathbf{G},\mathbf{x}) $ and also the masking mechanism that both ensures acyclicity and enables the use of the  copy mechanism. We have $ \mathbf{W}=\mathbf{W}^{\mathrm{raw}} + \mathbf{W}^{\mathrm{mask}} $. First, we define the unmasked logits, $\mathbf{W}^{\mathrm{raw}}= \mathbf{A}^{\mathrm{raw}}\circ  \mathbf{S}^{\mathrm{raw}} $:
\begin{align}
    \nonumber
    \mathbf{h}^{\mathrm{g}} &= \mathrm{RelGCN}(\mathbf{G};\theta)  \in \mathbb{R}^{m \times d} \\
 \nonumber
    \mathbf{A}^{\mathrm{raw}} &= \mathrm{BiAffine}^{\mathrm{align}}( \mathbf{h}^{\mathrm{token}}, \mathbf{h}^{\mathrm{g}}\circ \mathbf{h}^{\mathrm{end}} ;\phi) \\
    \nonumber
    \mathbf{S}^{\mathrm{raw}} &= \mathrm{BiAffine}^{\mathrm{segment}}(  \mathbf{h}^{\mathrm{g}}, \mathbf{h}^{\mathrm{g}}\circ \mathbf{h}^{\mathrm{end}}  ;
    \phi) 
\end{align}
where $\mathrm{RelGCN}$ is a relational graph convolutional network~\cite{Schlichtkrull2018ModelingRD} that takes  an AMR graph $\mathbf{G}$  and produces embeddings of its nodes informed by their neighbourhood in $\mathbf{G}$.  $\mathbf{h}^{\mathrm{end}} \in \mathbb{R}^{1 \times d}$ is the trainable embedding of the terminal node, and $\mathbf{h}^{\mathrm{token}}  \in \mathbb{R}^{n \times d} $ is the BiLSTM encoding of a sentence from Section~\ref{sec:seg_concept_generation}.

%

The masking also consists of two parts, the alignment mask and the segmentation mask,  $\mathbf{W}^{\mathrm{mask}}=\mathbf{A}^{\mathrm{mask}}\circ  \mathbf{S}^{\mathrm{mask}}  $.
If a node is copy-able from at least one token, the alignment mask prohibits alignments from other tokens by setting the corresponding components $\mathbf{A}^{\mathrm{mask}}_{ij}$ to $-\infty$. 

Acyclicity is ensured by setting $\mathbf{S}^{\mathrm{mask}}$ so that generation order with circles will get negative infinity in Equation~\ref{eq:reg}.
While there may be more general
ways to encode acyclicity~\cite{martins2009concise},
we simply perform a depth-first search (DFS) from the root node\footnote{We use
lexicographic ordering of edge labels in DFS.}
and permit an edge from node $i$ and $j$ only if $i$ precedes $j$ (not necessarily immediately) in the traversal. In other words,  
$\mathbf{S}^{\mathrm{mask}}_{ij}$ is set to $-\infty$ for edges $(i,j)$ violating this constraint.
The rest of components in $\mathbf{S}^{\mathrm{mask}}$ are set to 0. Note that this masking approach does not require changes in the optimization method.

\section{Parsing}
\label{sec:parsing}

While we relied on the latent variable machinery to train the parser, we do not use it at test time. In fact, the encoder $Q_\phi(\mathbf{O}|\mathbf{G},\mathbf{x})$ is discarded after training. At test time, the first step is to predict sets of concept nodes for every token using the concept identification model $P_\theta(\mathbf{v}|\mathbf{x},\mathbf{O})$ (as shown in Figure~\ref{fig:seg_lstm}). Note that the token-specific autoregressive models can be run in parallel across tokens. The second step is predicting relations between all the nodes, relying on the relation identification model $P_\theta(\mathbf{E}|\mathbf{x},\mathbf{O},\mathbf{v}) $. 

\section{Experiments} 

We experiment on LDC2016E25 (AMR2.0) and 
LDC2020T02 (AMR3.0).  \nocite{Cai2013SmatchAE}
%
The evaluation is  based on Smatch~\cite{Cai2013SmatchAE}, and the evaluation tool of~\newcite{Damonte2017AnIP}.  
We compare
our generation-order induction framework to pre-set segmentations, i.e., producing the segmentation on a preprocessing step.  We vary the segmentation methods while keeping the rest of the model identical to our full model (i.e., the same autoregressive model and the learned alignment). We provide ablation studies for our induction framework. We further provide visualization of the induced generation order, along with extra details, in Appendix. 

\vspace{1ex}
\noindent{\bf Rule-based Segmentation}  \label{sec:fixed-seg} \quad
We introduce a hand-crafted rule-based segmentation method, which relies on rules designed to handle specific AMR constructions. In particular, we use the hand-crafted segmentation system of  \newcite{lyu-titov-2018-amr}, or, more specifically, its re-implementation by \newcite{Zhang2019AMRPA}.  
Arguably, this can be thought of as an upper bound for how well an induction method can do.
This fixed segmentation can be  incorporated into our  latent-generation-order framework, so that the alignment  between concept nodes and the tokens will still be induced. This is achieved by fixing $\mathbf{S}$, while still inducing $\mathbf{A}$.

\vspace{1ex}
\noindent{\bf Greedy Segmentation}  \quad
We provide a greedy strategy for segmentation that serves as a deterministic baseline. Many nodes are aligned to tokens with the copy mechanism. We could force the unaligned nodes to join its neighbors. This is very similar to the forced alignment of unaligned nodes used in the transition parser of \newcite{Naseem2019RewardingST}. Again, the segmentation can be  incorporated into our  latent-generation-order framework by enforcing $\mathbf{S}$ and inducing $\mathbf{A}$.
See Appendix~\ref{append:greedy} for extra details about the strategy.

\vspace{1ex}
 \noindent{\bf Results}  \quad
In Table~\ref{table:sota}, we compare our models with recent AMR parsers~\cite{xu-etal-2020-improving,Cai2020AMRPV,cai-lam-2019-core,Zhang2019AMRPA,Naseem2019RewardingST,Lindemann2020FastSP,Lee2020PushingTL}, as well as \cite{lyu-titov-2018-amr}, which we build on, and \cite{Noord2017NeuralSP}, the earliest model which does not exploit any rules.  Overall, our model (`full') performs competitively, but lags behind scores reported by some of the very recent parsers.\footnote{Results from \citeauthor{Lee2020PushingTL} replace Roberta-large with Roberta-base in \citeauthor{Astudillo2020TransitionbasedPW}. With semi-supervised learning, \newcite{Lee2020PushingTL} achieved 81.3 Smatch score. } 
 However, except for a no-rule version of \newcite{Cai2020AMRPV}, all these models either use rules~\cite{Lee2020PushingTL} (see Section~\ref{sec:related}) or specialized pretraining~\cite{xu-etal-2020-improving}.

Both our VAE model and the rule-based segmentation achieve high concept identification scores~\cite{Damonte2017AnIP}.
The relation identification component is however weaker than, e.g., \cite{Cai2020AMRPV}. This may not be surprising, as  we, following ~\newcite{lyu-titov-2018-amr}, score edges independently, whereas \cite{Cai2020AMRPV} perform iterative refinement which is known to boost performance on relations~\cite{lyu-etal-2019-semantic}.  Also, we use BiLSTM encoders, which -- while cheaper to train and easier to tune --
is likely weaker than Transformer encoders used by \citeauthor{Astudillo2020TransitionbasedPW,Lee2020PushingTL} 
While these modifications, along with using extra pre-training techniques and data augmentation, may further boost performance of our model, we believe that our model is strong enough for our purposes, i.e. demonstrating that informative segmentation can be induced without relying on any rules.

\begin{table}[t!] 
   \begin{center} \setlength\tabcolsep{3pt} 
          \begin{tabular}{llccc} 
   \hline   
          & R &   Concept & SRL &  Smatch  \\\hline
    vNoord17$^\diamondsuit$  & - &  &  & 71.0 \\      
    Lyu18 &  +   & 85.9  &69.8&74.4\\
  Zhang19  & +    & 86  &71& 77.0\\
   Naseem19  & +    & 86 &72&75.5\\
   Cai19 & - & &  &   73.2 \\
   Lindemann20  & +  &   & & 76.8\\
   Lee20  & +    & 88.1  &78.2 & \bf 80.2\\
 Cai20: \\ 
 \quad w/ rules & +   & 88.1 & 74.2& \bf80.2\\
 \quad w/o rules & -   & 88.1 & 74.5&  78.7\\
 Xu20$^\diamondsuit$  &  - & 87.4 & 78.9 & \bf 80.2  \\
  
     \hline
      greedy  & -   & 87.5   \small $ \pm$ 0.1& 71.3 \small $ \pm$0.1 & 75.2 \small $ \pm$ 0.1\\
     rule    & + & \bf 88.7 \small $ \pm$ 0.2  &  73.6\small $ \pm$ 0.2  & 76.8 \small $ \pm$ 0.4\\
  full     &  - &  88.3 \small $ \pm$ 0.3 & 73.0 \small $ \pm$ 0.2& 76.1 \small $ \pm$ 0.2\\ \hline
        \end{tabular}
    \end{center}
    \vspace{-2ex}
	\caption{\label{table:sota} Scores with standard deviation on the AMR 2.0 test set.  digits.  The columns 'R' indicate if hand-crafted rules are used for segmentation, $\diamondsuit$ indicates that the system used specialized pretraining or self-training.  Our results are averaged over 4 runs. 
    }
\end{table}

\begin{table}[t!] 
   \begin{center} \setlength\tabcolsep{3pt} 
     \begin{tabular}{lccc} 
   \hline 
        Metric  &   Concept & SRL &  Smatch \\\hline
         greedy    &87.0& 71.5& 74.8 \\
        rule    & \bf 88.0 & 72.6 &\bf 75.8 \\
     full     &  87.8  &\bf 72.9 & 75.6 \\ \hline
        \end{tabular}
    \end{center}
    \vspace{-2ex}
	\caption{\label{table:sota3}  AMR 3.0 test set, averaged over 2 runs.
    }
\end{table}

Indeed, our approach beats the greedy baseline and approaches the rule-based system. The performance gap between the rule-based system and VAE is smaller on AMR 3.0 (0.2 Smatch), possibly because the rules were developed for AMR 2.0. 

 \noindent{\bf Alignment Analysis }  \quad
We analyzed the alignment induced by our full model and the model which uses rule-based segmentation. The alignments were evaluated at the level of individual concepts: if a subgraph was aligned to a token, all its concepts were considered aligned to that token. The evaluation was done on 40 sentences. The alignment error rates were 12\%, 15\% and 14\% for the full model, greedy methods and the rule-based method, respectively. This suggests that our method is able to induce relatively accurate alignments, and joint induction of alignments with segmentation may be beneficial, or, at the very least, not detrimental to alignment quality. 

 \noindent{\bf Ablations}   \quad
\begin{table}[t!] 
    \begin{center} 
        \begin{tabular}{lccc} 
            \hline     &   Concept & SRL &  Smatch \\\hline
          nothing learned    & 81.7  &62.6& 61.9\\
          segmentation learned    & 86.0  &69.1& 70.5\\
          alignment learned     & 87.6  &71.1& 74.4\\
          full  (all learned)     & \bf 88.3 &\bf 73.0&\bf 76.1\\ \hline
        \end{tabular}
    \end{center}
    \vspace{-2ex}
	\caption{\label{table:ablation_vae} Scores with different versions of latent segmentation on the AMR 2.0 test set, averaged over 2 runs
    }
    \vspace{-2ex}
\end{table}
To reconfirm that it is important to learn the segmentation and alignment, rather than to sample it randomly, we perform further ablations. In our parameterization, discussed in Section~\ref{sec:mask},  it is possible to set $\mathbf{A}^{\mathrm{raw}}=0 $ and/or $  \mathbf{S}^{\mathrm{raw}}=0$, which corresponds to sampling from the prior in training  (i.e. quasi-uniformly while respecting the constraints defined by masking) rather than learning them. We consider 4 potential options, from sampling everything uniformly to learning everything (as in our method). 
The results are summarized in Table~\ref{table:ablation_vae}. As expected, the full model performs the best, demonstrating that it is important to learn both alignments and segmentation. Interestingly, both  `segmentation learned' and `alignment learned' obtain reasonable performance, but the `nothing learned' model fails badly. 

\section{Related Work}\label{sec:related}
A wide range of approaches for AMR parsing have been explored, including graph-based models~\cite{Flanigan2014ADG,Werling2015RobustSG,lyu-titov-2018-amr,Zhang2019AMRPA}, transition-based models~\cite{Damonte2017AnIP,Ballesteros2017AMRPU}, grammar-based models~\cite{peng-etal-2015-synchronous,Artzi2015BroadcoverageCS,Groschwitz2018AMRDP,Lindemann2020FastSP} and neural autoregressive models~\cite{Konstas2017NeuralAS,Noord2017NeuralSP,Zhang2019BroadCoverageSP,Cai2020AMRPV,Xu2020ImprovingAP}. 


The majority of strong parsers rely on explicit graph segmentation in training. 
Typically, the segmentation is dealt with hand-crafted rules, with rule templates
developed by studying training set statistics and ensuring the necessary level of coverage. 
Alternatively, ~\newcite{Artzi2015BroadcoverageCS,groschwitz-etal-2017-constrained,Groschwitz2018AMRDP,Lindemann2020FastSP,peng-etal-2015-synchronous} using existing grammar formalisms to segment the AMR graphs. 
 \newcite{Astudillo2020TransitionbasedPW,Lee2020PushingTL} - while not not relying on graph recategorization rules -  use a rule system to `pack' and `unpack' nodes. In recent work,  strong results were obtained without using any explicit segmentation and alignment, relying on sequence-sequence models~\cite{Xu2020ImprovingAP,Cai2020AMRPV}, still the rules appear useful even with these strong models~\cite{Cai2020AMRPV}.  


More generally, outside of AMR parsing, differentiable relaxations of latent structure representations have received attention in NLP~\cite{kim2017structured,liu2018learning}, including previous applications of the perturb-and-MAP framework~\cite{Corro2019DifferentiablePS}. From a more general goal perspective -- inducing a segmentation of a linguistic structure -- our work is related to tree-substitution grammar induction~\cite{sima1995cient,cohn2010inducing}, the DOP paradigm~\cite{bod2003data} and unsupervised semantic parsing~\cite{poon2009unsupervised,titov2011bayesian}, though the methods used in that previous work are very different from ours.
 

\section{Conclusions} 
To eliminate hand-crafted segmentation systems used in previous AMR parsers, we cast the alignment and segmentation as generation-order induction. We propose to treat this generation order as a latent variable in a VAE framework. Our method outperforms a simple segmentation heuristic and approaches the performance of a method using rules designed to handle specific AMR constructions.   
Importantly, while the latent variable modeling machinery is used in training, the parser is very simple at test time. It tags the input words with AMR concept nodes with autoregressive models and then predicts relations between the nodes independently from each other.  

Vanilla sequence-to-sequence models are known to struggle with out-of-distribution generalization~\cite{lake18a,bahdanau2018systematic}, and, in the future work, it would be interesting to see if this   holds for AMR and if such more constrained and structured methods as ours can better deal with this more challenging but realistic setting.

\section*{Acknowledgments}
We thank the reviewers for their useful feedback and comments.
The project was supported by the European Research Council (ERC StG BroadSem 678254), the Dutch National Science Foundation (NWO VIDI 639.022.518) and Bloomberg L.P.

\bibliography{eacl2021}
\bibliographystyle{acl_natbib}

\newpage

$\,$

\newpage

\appendix
\newpage\section{Decoding}~\label{append:seg_decode}
We need to model the identification of the root node of the AMR graph. We specify the root identification as:
\begin{align}
    P_\theta(i|\mathbf{x},\mathbf{O},\mathbf{v}) &=  \frac{\exp (\langle \mathbf{h}^\mathrm{root} , \mathbf{h}^e_i\rangle) }{\sum_{j =0 }^{m-1} \exp (\langle \mathbf{h}^\mathrm{root} ,\mathbf{h}^e_j\rangle)} 
\end{align}
where $\mathbf{h}^\mathrm{root}$ is a trainable vector. Inspired by~\newcite{Zhang2019AMRPA}, who rely on AMR graphs being closely related to dependency trees, we first decode the AMR graph as a maximum spanning tree with log probability of most likely arc-label as edge weights. The reentrancy edges are added afterwards, if their probability is larger than $0.5$. We add at most 5 reentrancy edges, based on the empirical founding of~\newcite{Damonte2020rent}. 

\section{Concept Identification Detail}~\label{append:concept}
Now, we specify $P_\theta(\mathbf{v}_i|\mathbf{h}^{\mathrm{node}}_i) $ and $P_\theta(\mathbf{v}_m|\mathbf{h}^{\mathrm{tail}}_i)$ with a copy mechanism. Formally, we have a small set of candidate nodes $\mathcal{V}(\mathbf{x}_i)$  for each token $\mathbf{x}_i$, and a shared set of candidate nodes $\mathcal{V}^{\mathrm{share}}$, which contain $v_{copy}$. This, however, depends on the token, yet we are learning a latent alignment. During training, we consider all the union of candidate nodes from all possible tokens$ \mathcal{V}(\mathbf{v}_i) = \cup_{j: \mathbf{v}_i \in \mathcal{V}(\mathbf{x}_j)}    \mathcal{V}(\mathbf{x}_j)  $. We abuse notation slightly, and denote the embedding of node $i$ by $\mathbf{v}_i$. At training time, for node $\mathbf{v}_i$, we have \begin{align}
\mathbf{h}^c_i &= \mathrm{NN}^{\mathrm{node}}(\mathbf{h}^{\mathrm{node}}_i;\theta) \\
P_\theta(\mathbf{v}_i|\mathbf{h}^{\mathrm{node}}_i) 
=& \frac{ [[\mathbf{v}_i \in \mathcal{V}^{\mathrm{share}} ]]  \exp (\langle \mathbf{v}_i,\mathbf{h}^c_i \rangle) }{\sum_{v \in \mathbb{\mathbf{v}} } \exp (\langle v,\mathbf{h}^c_i\rangle)}  \nonumber\\
+&   \frac{ [[\mathbf{v}_i \in \mathcal{V}(\mathbf{v}_i)]]  \exp (\langle v_{copy},\mathbf{h}^c_i \rangle) }{\sum_{v \in \mathbb{\mathbf{v}} } \exp (\langle v,\mathbf{h}^c_i\rangle)} \nonumber \\
\times &   \frac{  \exp (\mathrm{S}(\mathbf{v}_i,\mathbf{h}^c_i )) }{\sum_{v \in  \mathcal{V}(\mathbf{v}_i)  } \exp (\mathrm{S}(v,\mathbf{h}^c_i )} 
\end{align}
where $\mathrm{NN}$ is a standard one-layer feedforward neural network, and $[[\ldots]]$ denotes the indicator function. $\mathrm{S}(v,\mathbf{h}^c_i )$ assigns a score to candidate nodes given the hidden state. To use pre-trained word embedding~\cite{pennington2014glove}, the representation of $v$ is decomposed into primitive category embedding $(\mathcal{C}(v)$\footnote{AMR nodes have primitive category, including string, number, frame, concept and special nodes (e.g. polarity).} and surface lemma embedding. The score function is then a biaffine scoring based on the embeddings and hidden states$(\mathcal{L}(v)$ $\mathrm{S}(v,\mathbf{h}^c_i )=\mathrm{Biaffine}(\mathcal{C}(v)\circ \mathcal{L}(v),\mathbf{h}^c_i;\theta)$. For the terminal nodes, we have:
\begin{align}
\mathbf{h}^t_i &= \mathrm{NN}^{\mathrm{node}}(\mathbf{h}^{\mathrm{tail}}_i;\theta)  \\
   P_\theta(\mathbf{v}_m|\mathbf{h}^{\mathrm{tail}}_i) &= \frac{ \exp (\langle \mathbf{v}_m,\mathbf{h}^t_i \rangle) }{\sum_{v \in \mathbb{\mathbf{v}} } \exp (\langle v,\mathbf{h}^t_i\rangle)} 
\end{align}

At testing time, we perform greedy decoding to generate nodes from each token in parallel until either terminal node or $T$ nodes are generated.
\section{Computing $\mathbf{B}$ and $\mathbf{A}^{\infty}$} \label{appendix:align_derivative} 
We obtain $\mathbf{B}$ by having:
\begin{align}
   \mathbf{B} = \mathbf{A} [\mathbf{S}_{:,:m}+\mathrm{Diag}(\mathbf{S}_{:,m})]^T;
\end{align}
where $\mathbf{S}_{:,:m}$ takes the submatrix of  $\mathbf{S}$, excluding the last column, and $\mathrm{Diag}(\mathbf{S}_{:,m}$ is the diagonal matrix whose diagonal entries are the last column of $\mathbf{S}$. Intuitively, $[\mathbf{S}_{:,:m}+\mathrm{Diag}(\mathbf{S}_{:,m})]$ can be thought as a Markov transition matrix that passes down the alignment along the generation order, but keeps the alignment mass if the node will generate $\emptyset$. We truncate the transition at $T=4$, as we do not expect a subgraph containing more than 4 nodes.

To obtain $\mathbf{A}^{\infty}$, we observe $\mathbf{A}^{\infty}$ should obey the following self-consistency equation:
\begin{align}
   \mathbf{A}^{{\infty}} &=  \mathbf{A}^{{\infty}} \mathbf{S}_{:,:m} +  \mathbf{A} \label{eq:edge_align}
\end{align}
This means, node $j$ is generated from token $k$ iff node $i$ is is generated from token $k$ and node $i$ generates node $j$ or node  $j$ is directly generated from token $k$. This $\mathbf{A}^{\infty}$ can be computed  by initializing $ \mathbf{A}^{{\infty}}=\mathbf{A}$, and repeating Equation~\ref{eq:edge_align} as assignment for $T=4$ times. Intuitively, the $ \mathbf{A}^{{\infty}}$ alignment is passed down along the generation order, while keeping getting alignment mass from the first node alignment. As a result, all nodes get assigned an alignment. As an alternative motivation, the above algorithmic assignment works as a truncated power series expansion of self-consistency equation solution $ \mathbf{A}^{{\infty}} = [I-\mathbf{S}_{:,:m}]^{-1}\mathbf{A} $.

\section{Ablation on Stochastic Softmax}
Our full model uses the Straight-Through (ST) gradient estimator and the Free Bits trick with $\lambda=10$~\cite{Kingma2017ImprovedVI}.\footnote{The Free Bits trick is used to prevent `the posterior collapse'~\cite{Kingma2017ImprovedVI}. In other words, we use $\max (\lambda, \mathrm{KL}(\mathcal{G}(  \mathbf{W},1)||\mathcal{G}(0,1)) )$ for the KL divergence regularizer.} 
We perform an analysis of different variations of the stochastic softmax:  (1)  the soft stochastic softmax is the original one with the entropic regularizer (see Section~\ref{sec:soft});  (2) the rounded stochastic softmax, which selects the highest scored next node from each tokens and concept nodes based on the soft stochastic softmax;\footnote{Such rounding does not provide any guarantee of being a valid generation order, but serves as a baseline. In general, a threshold function (at 0.5) can be applied if the constraints have no structure.} (3) our full model with the ST estimator.  All those models use Free Bits ($\lambda=10$),  while for `no free bits' $\lambda=0$. 
\begin{table}[t!] 
    \begin{center} 
        \begin{tabular}{llll} 
            \hline  Metric   &   Concept & SRL &  Smatch \\\hline
   no free bits  &  83.5  & 66.3 &  66.1 \\ 
           soft    & 84.9 & 68.1& 70.3\\
          rounding     & 87.7  &71.8& 74.5\\
          straight-through     &\bf  88.3 &\bf 73.0&\bf 76.1\\ \hline
        \end{tabular}
    \end{center}
    \vspace{-2ex}
	\caption{\label{table:ablation_ss} Scores with different versions of latent segmentation on the AMR 2.0 test set. Scores are averaged over 2 runs
    }
\end{table}
As we can see in Table~\ref{table:ablation_ss},  there is a substantial gap between using structured ST and the two other versions. This illustrates the need for exposing the parsing model to discrete structures in training. Also, the Free Bits trick appears crucial as it prevents the (partial) posterior collapse in our model. We inspected the logits after training and observed that, without free-bits, the learned $\mathbf{W}$ are very small, in the $[-0.01,+0.01]$ range.   

\section{Greedy Segmentation} \label{append:greedy}
We present a greedy strategy for segmentation that serves as a deterministic baseline. This greedy segmentation can be used in the same way as the rule-based segmentation by setting $\mathbf{S}^\mathrm{mask}$. 

Many nodes are aligned to tokens with the copy mechanism. We could force the unaligned nodes to join their neighbors. This is very similar to the forced alignment of unaligned nodes used in the transition parser of \newcite{Naseem2019RewardingST}. 
We traversal the AMR graph the same way as we do when we produce the masking (Section~\ref{sec:mask}).  During the traversal, we greedily combine subgraphs until one of the constraints is violated: (1) the combined subgraph will have more than 4 nodes; (2) the combined subgraph will have more than 2 copy-able nodes. We present the algorithm recursively (see Algorithm~\ref{alg:heu}).
\begin{algorithm}[ht!] 
\SetAlgoLined
\KwIn{ graph $\mathbf{G}$,  node index $i$}
\KwResult{ segmentation $\mathbf{S}$, $n$,  $z$, $k$}
\vspace{1ex}
  $\mathbf{S} = \mathbf{0} $, $k=i$, $n=1$, $z=\mathbf{z}_i$\;
  \ForAll{$j \in \mathrm{Child}[i]$}{
    \If {$j$ $ \mathrm{not visited}$} {
    $\mathbf{S}', n', z', k' = \mathrm{Greedy}(\mathbf{G},j)$ \;
        $\mathbf{S}=\mathbf{S} + \mathbf{S}'$\;
        \If {$n+n'\leq T \land z'+z\leq 1$ } {
        $\mathbf{S}_{kj}=1$, $n = n + n'$, $z=z+z'$ $k =k'$\; 
        }
    }
}
 \caption{Greedy Segmentation \label{alg:heu}}
\end{algorithm}
Variable $\mathbf{z}_i $ indicates whether node $i$ is copy-able and $T=4$ represent the maximum subgraph size; $n$ denotes the current subgraph size;  $z$ indicates whether the current subgraph contains a copy-able node; $k$ is the last node in the current subgraph, which is used to generate to future nodes in a subgraph. The condition $n+n'\leq T \land z'+z\leq 1$ determines whether we combine the current subgraph rooted at node $i$ and the subgraph rooted at node $j$. Running the algorithm on an AMR graph and the root index will get us the entire segmentation. This greedy method does not require any expert knowledge about AMR, so this should serve as a baseline.

\section{Visualizing Generation Order}\label{append:seg_visual}

\begin{figure}
    \centering  
    \includegraphics[width=\linewidth]{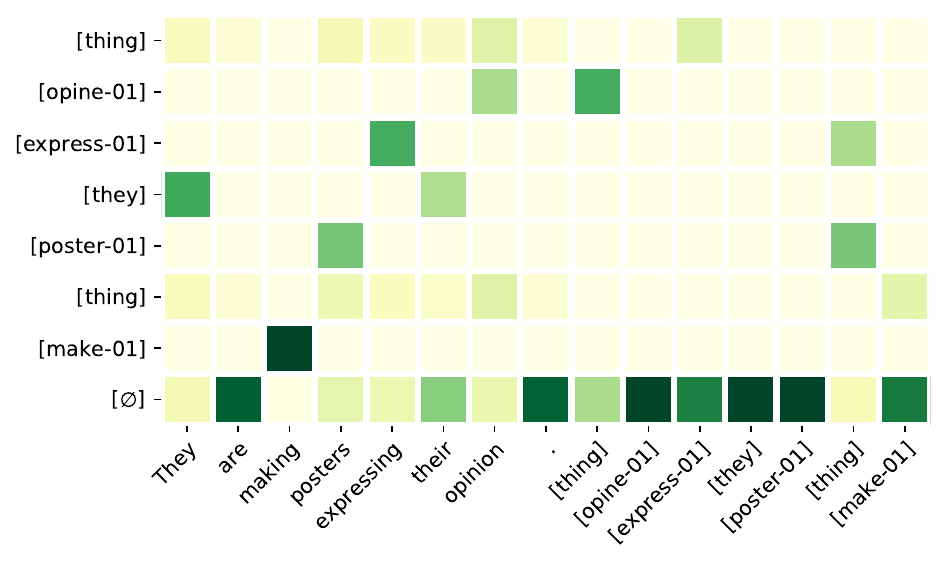} 
    \caption{Example of Soft Stochastic Softmax Latent Generation Order.} \label{fig:soft}
    \centering
    \includegraphics[width=\linewidth]{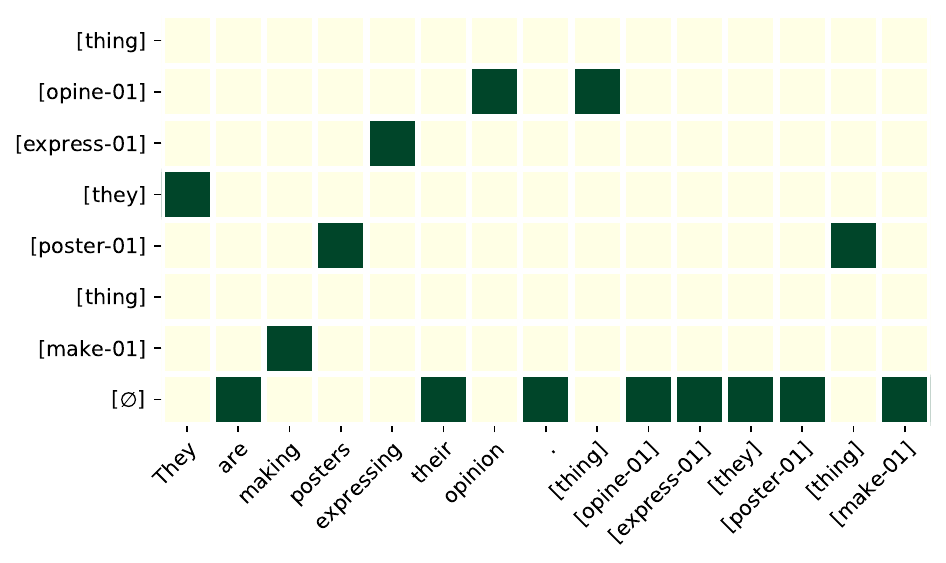} 
    \caption{Example of Rounded Stochastic Softmax Latent Generation Order.} \label{fig:binary}
    \centering
    \includegraphics[width=\linewidth]{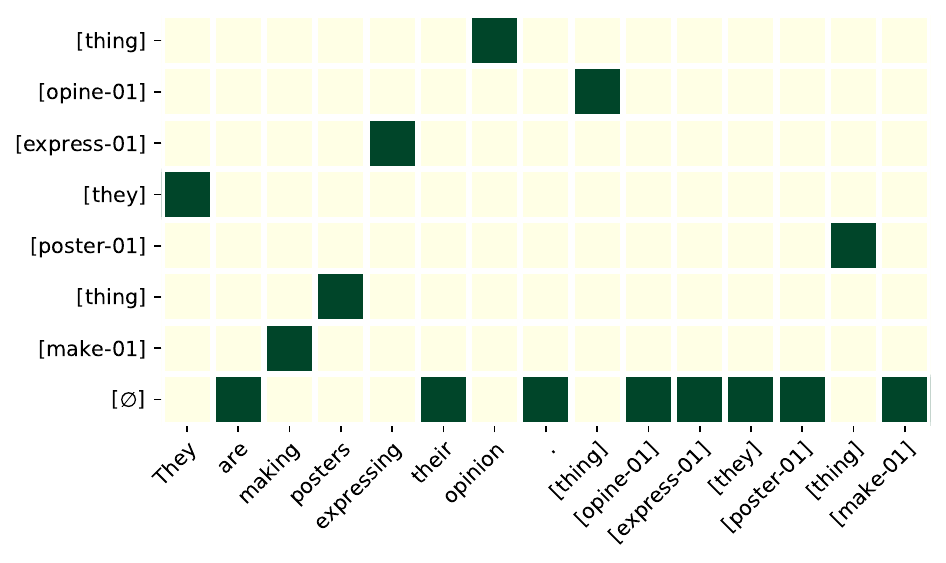} 
    \caption{Example of Hard (straight-through) Stochastic Softmax Latent Generation Order.} \label{fig:hard}
    \centering
    \includegraphics[width=\linewidth]{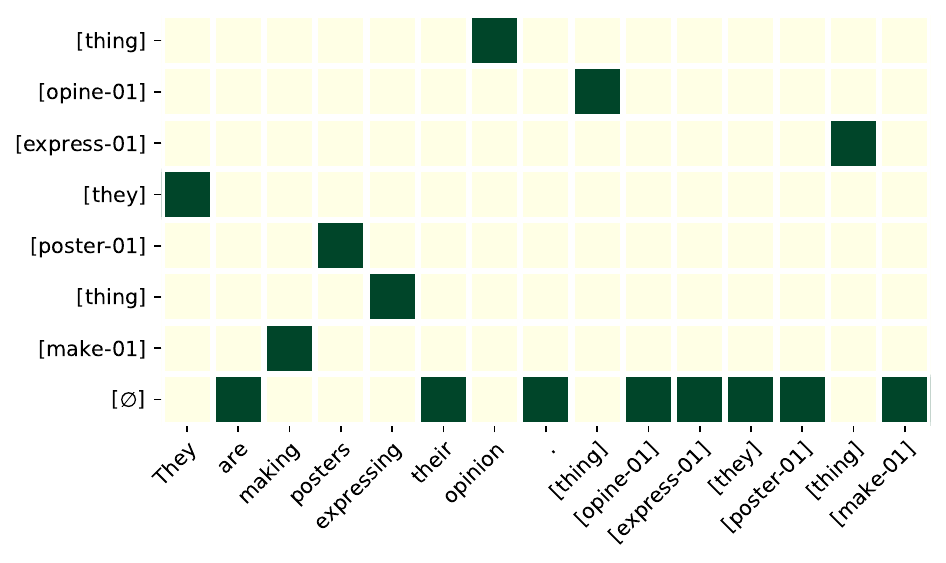}
    \caption{Example of Rule-Segmentation Stochastic Softmax Latent Generation Order.}  \label{fig:rule}
\end{figure}
In Figures~\ref{fig:soft},~\ref{fig:binary},~\ref{fig:hard},~\ref{fig:rule},\footnote{Incidentally, the greedy segementation produces the same segmentation as rule-based in this example.} we present one example of the induced learned for our three variations of stochastic softmax, and one with rule-based segmentation. The nodes are represented in []. Their gold AMR is:
\begin{lstlisting}[  basicstyle=\small]
(m / make-01
    :ARG0 (t / they)
    :ARG1 (t2 / thing
        :ARG2-of (p / poster-01)
        :ARG0-of (e / express-01
            :ARG1 (t3 / thing
                :ARG1-of (o / opine-01
                    :ARG0 t)))))
\end{lstlisting}
As we can see, the standard stochastic softmax indeed produces soft latent structure that might result in large training/testing gap. Furthermore, the rounding strategy does not satisfy the constraint that every concept node can only be generated from one token or another concept node (i.e. [poster-01] is generated twice, and [thing] is never generated.). Meanwhile, the straight through stochastic softmax produce a valid generation order. In Appendix~\ref{append:valid}, we will show the validity formally. It is worth to note that our learned generation order differs from the rule based one. When producing the rule-based segmentation, `(t2 / thing  :ARG0-of (e / express-01 )' took precedence over `(t2 / thing  :ARG2-of (p / poster-01)' due to the order over traversal edges. The learned model, however, figured out that the poster is the thing.

\section{Hyper-Parameters }
We use RoBERTa-large~\cite{Liu2019RoBERTaAR} from ~\newcite{Wolf2019HuggingFacesTS} for contextualised embeddings before LSTMs.  BiLSTM for concept identification has 1 layer, and BiLSTM for relation identification has 2 layers. Both have hidden size 1024. Their averaged representation is used for alignment. RelGCN used 128 hidden units and 1 hidden layer (plus one input layer and output layer). Relation identification used 128 hidden units. The LSTM for the locally auto-regressive model is one layer with 1024 hidden units. Adam~\cite{Kingma2015AdamAM} is used with learning rate $3e-4$ and beta=$0.9,0.99$. Early stopping is used with maximum 60 epochs of training. Dropout is set at 0.33. Those hyper-parameters are selected manually, we basically followed the standard model size as in ~\cite{lyu-titov-2018-amr,Zhang2019AMRPA}. We will release the code based on the AllenNLP framework~\cite{Gardner2018AllenNLPAD}.

\section{Pre-and-Post processing}
We follow ~\newcite{lyu-titov-2018-amr} for pre-and-post processing. We use CoreNLP~\cite{manning-EtAl:2014:P14-5} for tokenization and lemmatization. The copy-able dictionary is built with the rules based on string matching between lemmas and concept node string as in~\newcite{lyu-titov-2018-amr}. 

For post-processing, wiki tags are added after the named entity being produced in the graph via a look-up table built from the training set or provided by CoreNLP. We also collapse nodes that represent the same pronouns as heuristics for co-reference resolution.
    
\section{Proof of Proposition~\ref{prop:order-converge}} \label{prof:order-converge}
We prove Proposition~\ref{prop:order-converge} based on the Bregman method~\cite{Bregman1967TheRM}. The Bregman's  method solves convex optimization with a set of linear equalities, the setting is as follows:
\begin{align} \label{eq:breg}
\min_{x \in \Bar{\Omega}}\ &\ F(x) &\text{s.t.}\ & Ax=b,\,  
\end{align}
where $F$ is strongly convex and continuously differentiable. Note that $A$ is not our alignment, but denotes a matrix that represents constraints. Two important ingredients are Bregman's divergence $ D_F(x,y) =  F(x) - F(y) - \langle \nabla F(y) ,x-y \rangle  $, and Bregman's projection: $ P_{\omega,F}(y) = \argmin_{x\in \omega} D_F(x,y)  $, where $\omega$ represents constraint. Now, the Bregman's method works as:
\begin{algorithm}[ht]
\SetAlgoLined
 pick $ y^0 \in \{y\in \Omega |\nabla F(y) = uA,  u \in \mathbb{R}^m\}$\;
  \For{$t\gets1$ \KwTo $\infty$}{
 $ y_0^t \gets y^{t-1} $ \;
   \For{$i\gets1$ \KwTo $m$}{
     $ y_i^t \gets P_{A_ix=b_i,F}(y^t_{i-1})$ \;
    }
    
 $ y^t \gets y_m^{t} $ \;
 }
 \caption{Bregman's method for solving convex optimization over linear constraints}
\end{algorithm}
Intuitively, Bregman's method iteratively performs alternating projections w.r.t. each constraint. After each projection, the score $F$ is lowered by the construction of Bregman's projection. Such alternating projections eventually converge, and with careful initialization solve the optimization problem.  
\begin{theorem}[\protect{\citealt{Bregman1967TheRM}}]  \label{tm:bregman}
$ \lim_{t\rightarrow \infty}   y^t  $ solves the optimization problem~\ref{eq:breg}.
\end{theorem}
\begin{proof}[Proof of Proposition \ref{prop:order-converge}]
We show Proposition \ref{prop:order-converge} by showing the Algorithm defined by equations ~\ref{eq:init}, ~\ref{eq:column}, ~\ref{eq:fix} and  ~\ref{eq:row} implements Bregman's method. Then, Proposition \ref{prop:order-converge}  follows from Theorem \ref{tm:bregman}.

Now, we build Bregman's method for our optimization problem~\ref{eq:reg}. For simplicity, we focus on the linear algebraic structure, but do not strictly follow the standard matrix notation. We have $\mathbf{O}$ as variable, and $F(\mathbf{O}) = - \langle \widetilde{ \mathbf{W}}, \mathbf{O} \rangle + \tau \langle \mathbf{O}, \log \mathbf{O} -1 \rangle$\footnote{This regularizer differs from the original one in ~\ref{eq:reg} by a constant $m+n$, due to the constraints. So, the optimization problem is equivalent.}. For initialization, we have $\nabla F( \mathbf{O}) =  -  \widetilde{ \mathbf{W}} +\tau  \log \mathbf{O}$. Take $u$ = 0, we have $\mathbf{O}^{(0)}= \exp (\frac{ \widetilde{ \mathbf{W}}}{\tau}) \iff \log \mathbf{O}^{(0)} = \frac{ \widetilde{ \mathbf{W}}}{\tau} $. This corresponds to the initialization step as in our Equation~\ref{eq:init}. Then, we iterate through constraints to perform Bregman's projection. First, the column normalization constraints $ \forall j < m, \sum_{i=0}^{n+m-1} \mathbf{O}_{ij} =1$. Take a $j<m$, we need to compute $P_{\sum_{i=0}^{n+m-1} \mathbf{O}_{ij} =1,F}(\mathbf{O}^(t))$. A very important property is that our $F( \mathbf{O})  =\sum_{ij} f_{ij}(\mathbf{O}_{ij})$, where $f_{ij}(\mathbf{O}_{ij}) = - \widetilde{ \mathbf{W}}_{ij}\mathbf{O}_{ij} +\tau \mathbf{O}_{ij} (\log \mathbf{O}_{ij} -1) $. Moreover, $D_F(x,y) = 0 \iff x=y$. Therefore, for variables that are not involved in the constraints, they are kept the same. To simplify notation, we extend the domain of $F$ to parts of the variable. e.g., $F(\mathbf{O}_{:,j} )=\sum_{i} f_{ij}(\mathbf{O}_{ij})$. Now, let us focus on column $j$, we have:
\begin{align}
   & \argmin_{x: \sum_i x_i =1 }  F(x) - F(\mathbf{O}_{:,j}) - \langle \nabla F(\mathbf{O}_{:,j} ) ,x-\mathbf{O}_{:,j} \rangle  \\
    = & \argmin_{x: \sum_i x_i =1 }   - \langle \widetilde{ \mathbf{W}}_{:,j}, x \rangle + \tau \langle x, \log x -1 \rangle   \nonumber \\
    &- \langle \nabla F(\mathbf{O}_{:,j} ) ,x \rangle \\
    = & \argmin_{x: \sum_i x_i =1 }   - \langle \widetilde{ \mathbf{W}}_{:,j}, x \rangle + \tau \langle x, \log x -1 \rangle   \nonumber\\
    &- \langle-  \widetilde{ \mathbf{W}}_{:,j} +\tau  \log \mathbf{O}_{:,j} ,x \rangle  \\
    = & \argmin_{x: \sum_i x_i =1 } \tau \langle x, \log x -1 \rangle   + \langle \tau  \log \mathbf{O}_{:,j} ,x \rangle   \\
    = & \argmin_{x: \sum_i x_i =1 } \langle x, \log x -1 \rangle   + \langle  \log \mathbf{O}_{:,j} ,x \rangle   \\
  =& \mathrm{Softmax}(\log \mathbf{O}_{:,j} )
\end{align}
since when iterating over these mutually non-overlapping constraints, the non-focused variables are always kept the same. It is hence equivalent to computing them in parallel, which is expressed in our column normalization step ~\ref{eq:column}. Similarly, we can derive row normalization step ~\ref{eq:row}. Therefore, our algorithm is an implementation of Bregman's method, and Proposition \ref{prop:order-converge} follows from Theorem \ref{tm:bregman}.
\end{proof}

\section{Generation Order is Discrete by LP} \label{append:valid}
If $\mathbf{O}( \widetilde{ \mathbf{W}},0) $ is integral valued, it  belongs to $\mathcal{O}$ by definition. In most cases, there is no guarantee that the linear programming in the relaxed space yields a solution that is also an integer.  However, in our cases, we have the following result:
\begin{prop}\label{prop:uni}
With probability 1, a unique $\mathbf{O}( \widetilde{ \mathbf{W}},0) \in \{0,1\}^{(n+m) \times (m +1)} $, where $\mathbf{O}( \widetilde{ \mathbf{W}},0)  $ is defined in Equation~\ref{eq:reg}.
\end{prop}
Intuitively, this is a generalization of a classical result about perfect matching on bipartite graph~\cite{Integer-Programming}. To prove this, we need the following theorems from integer linear programming.
\begin{theorem}[{\citealt[page 130,133]{Integer-Programming}}] l\label{thm:uni}
Let $A$ be an $q \times p$ integral matrix. For all integral vectors $ d, l, u  $ and  $c \in \mathbb{R}^p$,  $ \max\{\langle c,x\rangle  :  Ax = d, l \leq x \leq u\}$ is attained by an integral vector x  if and only if $ A$ is totally unimodular.\footnote{$A$ is totally unimodular if every square submatrix has determinant
$0,\pm 1$. We combined a few theorems and definitions from ~\newcite{Integer-Programming} into this theorem.}
\end{theorem}
Note that this theorem does not claim at all the solution is integer, nor that it is unique. However, one should understand this limitation as some degenerate case of $c$. However, a total unimodular matrix does characterize the convex hull of its integral points. To prove this, we need an additional lemma. 
\begin{lemma}[{\citealt[page 21]{Integer-Programming}}] l\label{lemme:convex_hull}
Let $S \in \mathbb{R}^n$ and $c \in \mathbb{R}^n$. Then $\sup \{\langle c,s \rangle  :s\in S \} = \sup \{\langle c,s \rangle  :s\in \mathrm{Conv}(S) \}$. Furthermore, the supremum of $\langle c,s \rangle $ is attained over $S$ if and only if it is
attained over $\mathrm{Conv}(S)$.
\end{lemma}
where $\mathrm{Conv}(S)$ is the convex hull of $S$. Now we have the following proposition:
\begin{prop}\label{prop:uni_convex_hull}
Let $A$ be an $q \times p$ integral matrix. For all integral vectors $ d, l, u  $ ,and  $c \in \mathbb{R}^p$ such that $\{ x \in \{0,1\}^p |  Ax = d, l \leq x \leq u\} $ is a finite set,  $\{ x \in \mathbb{p} |  Ax = d, l \leq x \leq u\} = \mathrm{Conv}(\{ x \in \{0,1\}^p |  Ax = d, l \leq x \leq u\})$  if and only if $ A$ is totally unimodular.
\end{prop}
In other words, we know the LP relaxation is the convex hull.
\begin{proof}
By Theorem~\ref{thm:uni}, $ A$ is totally unimodular is equivalent to maximum is attained by an integer solution. Clearly, the LP relaxation contains the convex hull. So, we only need to show that the LP relaxation does not contain any more points. Now suppose the LP relaxation contains another point $x'$ that's not in the convex hull. Since, we restrict our discussion on finite set of integer, both the $\{x'\}$ and the convex hull is closed set. Then by the separation theorem, we have a vector $c$ s.t. $\langle c,x' \rangle >  \langle c,x \rangle \forall x \in \mathrm{Conv}(\{ x \in \{0,1\}^p |  Ax = d, l \leq x \leq u\}) $,  which contradicts Lemma~\ref{lemme:convex_hull}.
\end{proof}

\begin{theorem}[\protect{\citealt[page 133,134]{Integer-Programming}}]\label{tm:bi}
A $0,\pm1$ matrix $A$ with at most two nonzero elements in
each column is totally unimodular if and only if rows of $A$ can be partitioned into two sets, red and blue, such that the sum of the red rows minus the sum of the blue rows is a vector whose entries are $0,\pm1$ (admits row-bicoloring).
\end{theorem}
Our $ \mathbf{O} $ should be the column vector $x$, and constraints should be represented by a matrix $A$. In particular, we view $  \mathbf{O}  $ as a column vector, but still access the item by $\mathbf{O}_{ij}$.\footnote{Alternatively, one could have a  vector $x $ and $x_{ i(m+1)+j} = \mathbf{O}_{ij} $. However, this will gets clumsy.} The matrix $A \in \{0,\pm 1\}^{(m+ (m+n)) \times ((n+m) (m+1))}$. $A_{:,ij}$ denotes the constraints involving $\mathbf{O}_{ij}$ . The first $m$ rows of $A$ correspond to $  \forall  j < m, \sum_{i=0}^{n+m-1} \mathbf{O}_{ij} =1 $, and the remaining $m+n$ rows correspond to $\forall i, \sum_{j=0}^{m} \mathbf{O}_{ij} =1 $. Therefore, we have $\forall k <m , j < m , i, A_{k,ij}= \delta_{j,k}$ and $\forall k\geq m , j  , i A_{k,ij}=  \delta_{i,k-m}$, else $A_{k,ij}= 0$, where $ \delta_{j,k} = [[j==k]]$. We have the linear constraints in standard form as $A\mathbf{O} = \mathbf{1}$.
\begin{lemma}
The $A$ defined above is totally unimodular.
\end{lemma}
\begin{proof}
First, we show $A$ admits row-bicoloring. We color the first $m$ rows red, and remaining $n+m$ rows blue. The sum of red rows is: $R_{ij}=\sum_{k=0 }^{m-1} A_{k,ij} =\sum_{k=0 }^{m-1}\delta_{j,k}= [[j<m]] $ and the sum of blues is $B_{ij}=\sum_{k=m }^{2m+n-1} A_{k,ij}= \sum_{k=m }^{2m+n-1} \delta_{i,k-m} = 1$. Therefore,  $R_{ij} - B_{ij} = [[j==m]] \in \{0,\pm1\}$, and $A$ admits a row-bicoloring. Since $A$ has only $0,\pm 1$ value, and one variable in $\mathbf{O}$ at most participates in two constraints (incoming and outgoing), by Theorem~\ref{tm:bi}, $A$ is totally unimodular. 
\end{proof}
Now, we prove Proposition~\ref{prop:uni}.
\begin{proof} \label{prof:uni}
We have $A$ being totally unimodular. We have  $c=\widetilde{W}$ , $l=0,u=1$, by Theorem~\ref{thm:uni}, the LP solutions contain an integer vector. Since the Gumbel distribution has a positive and differentiable density,  by \cite[Proposition~3]{Paulus2020GradientEW},  $ \argmax_{  \mathbf{O} \in \mathcal{O}} \langle \widetilde{ \mathbf{W}}, \mathbf{O} \rangle $  yields a unique solution with probability $1$. Clearly, this solution is the only integer solution in our LP solutions. Now, suppose another non-integer solution exists. We know the linear programming domain is the convex hull by Proposition~\ref{prop:uni_convex_hull}. Clearly, another integer solution exists, which contradicts the uniqueness of the integer solution. Hence, the  $\mathbf{O}( \widetilde{ \mathbf{W}},0)$ yields a unique integer solution with probability $1$.
\end{proof}

\end{document}